\documentclass[letterpaper, 10 pt, conference]{ieeeconf}  

\IEEEoverridecommandlockouts                              

\usepackage{amsmath} 
\usepackage{amssymb}  
\usepackage[dvipsnames]{xcolor}

\usepackage{mathtools}

\usepackage{booktabs}
\usepackage{caption}
\usepackage{subcaption}

\usepackage{tikz}
\usetikzlibrary{shapes,arrows}
\tikzstyle{block} = [draw, fill=white!20, rectangle, 
minimum height=3em, minimum width=6em]
\tikzstyle{sum} = [draw, fill=white!20, circle, node distance=1cm]
\tikzstyle{input} = [coordinate]
\tikzstyle{output} = [coordinate]
\tikzstyle{pinstyle} = [pin edge={to-,thin,black}]

\newcommand{\SP}{\hspace{2pt}}
\newcommand{\E}{\mathbb{E}}
\newcommand{\R}{\mathbb{R}}

\newcommand{\bx}{{\bf x}}
\newcommand{\by}{{\bf y}}
\newcommand{\bu}{{\bf u}}
\newcommand{\bw}{{\bf w}}

\newcommand{\bR}{{\bf R}}
\newcommand{\bP}{{\bf  P}}

\newcommand{\A}{{\bf A}}
\newcommand{\B}{{\bf B}}
\newcommand{\C}{{\bf C}}
\newcommand{\F}{{\bf F}}
\newcommand{\G}{{\bf G}}

\newcommand{\norm}[1]{\left\lVert#1\right\rVert}

\newcommand\vertarrowbox[3][2ex]{%
	\begin{array}[t]{@{}c@{}} #2 \\
		\left\uparrow\vcenter{\hrule height #1}\right.\kern-\nulldelimiterspace\\
		\makebox[0pt]{\scriptsize#3}
	\end{array}%
}
\usepackage[ruled,vlined]{algorithm2e}
\newtheorem{thm}{Theorem}
\newtheorem{rem}{Remark}

\newtheorem{defn}{Definition}
\newtheorem{cor}{Corollary}
\newtheorem{asm}{Assumption}
\newtheorem{exmp}{Example}

\title{\LARGE \bf  Robust Learning of  Recurrent Neural Networks in Presence of Exogenous Noise}

\author{Arash Amini$^{1}$, Guangyi Liu$^{1}$ and Nader Motee$^{1}$
	\thanks{$^{1}$A.Amini, G.Liu, and N.Motee are with the Department of Mechanical Engineering and Mechanics, Lehigh University, Bethlehem, PA 18015, USA.
		{\tt\small (a.amini,gliu,motee)@lehigh.edu}}%
}

\begin{document}

	\maketitle
	\thispagestyle{empty}
	\pagestyle{empty}

	\begin{abstract} 
	
	Recurrent Neural networks (RNN) have shown promising potential for learning dynamics of sequential data. However, artificial neural networks are known to exhibit poor robustness in presence of input noise, where the sequential architecture of RNNs exacerbates the problem. In this paper, we will use ideas from control and estimation theories to propose a tractable robustness analysis for RNN models that are subject to input noise.
    The variance of the output of the noisy system is adopted as a robustness measure to quantify the impact of noise on learning. It is shown that the robustness measure can be estimated efficiently using linearization techniques. Using these results, we proposed a learning method to enhance robustness of a RNN with respect to exogenous Gaussian noise with known statistics. Our extensive simulations on benchmark problems reveal that our proposed methodology significantly improves robustness of recurrent neural networks.   
	\end{abstract}

	\section{INTRODUCTION}
	Recurrent models have shown promising performance in  various applications over the past decades, due to their special design to handle sequential data, including image recognition \cite{vinyals2015show}, linguistics \cite{hannun2014deep}, and robotics \cite{liu2020reinforcement},\cite{harvey1994seeing}. In general, Recurrent neural networks (RNN) can be considered as a complex nonlinear dynamical system \cite{goodfellow2016deep},\cite{miller2018stable}, where their properties can be investigated using the existing tools developed for control and dynamical systems  \cite{pascanu2013difficulty}. The key difference between the two is that with RNNs  we are approximating the trainable parameters using a given set of sequential inputs to recover (learn) their underlying dynamics. In contrast, in control and dynamical systems, a model is already given and the objective is to control the dynamic behavior of the system using a set of sequential inputs \cite{anderson2007optimal}. The challenge of training RNN models for achieving an acceptable accuracy has been studied extensively  \cite{pascanu2013difficulty},\cite{hochreiter2001gradient}. There are different architectures proposed to improve the learning accuracy of RNNs, such as Long Short-term Memory (LSTM) \cite{hochreiter1997long} and Gated recurrent unit (GRU) \cite{cho2014learning}.

	The structure of a recurrent neural network is illustrated in Figure \ref{fig:RNN_unroll} which was first introduced in 1980's \cite{rumelhart1986learning,elman1990finding}.  The main objective for this class of neural networks is to learn time series and explore the correlation between different states throughout time. Unrolling a recurrent model, as shown in Figure \ref{fig:RNN_unroll}, yields a nonlinear dynamical system. As a result, techniques intended for dynamical systems and control are being used more frequently to analyze this branch of artificial neural networks. The authors of \cite{hardt2016gradient} show that  stable dynamical system is learnable with gradient descents methods. In  \cite{miller2018stable}, a notion of stability \cite{jin1994absolute} is used to answer questions regarding stable training of recurrent models and whether one can approximate a recurrent model with a deep neural network. The stability of RNN models is further investigated in \cite{manek2020learning} by applying the notion of stability in the sense of Lyapunov. 
	
    \begin{figure}[t]
    \centering
    \includegraphics[width=0.45\textwidth,trim=10 50 0 50, clip]{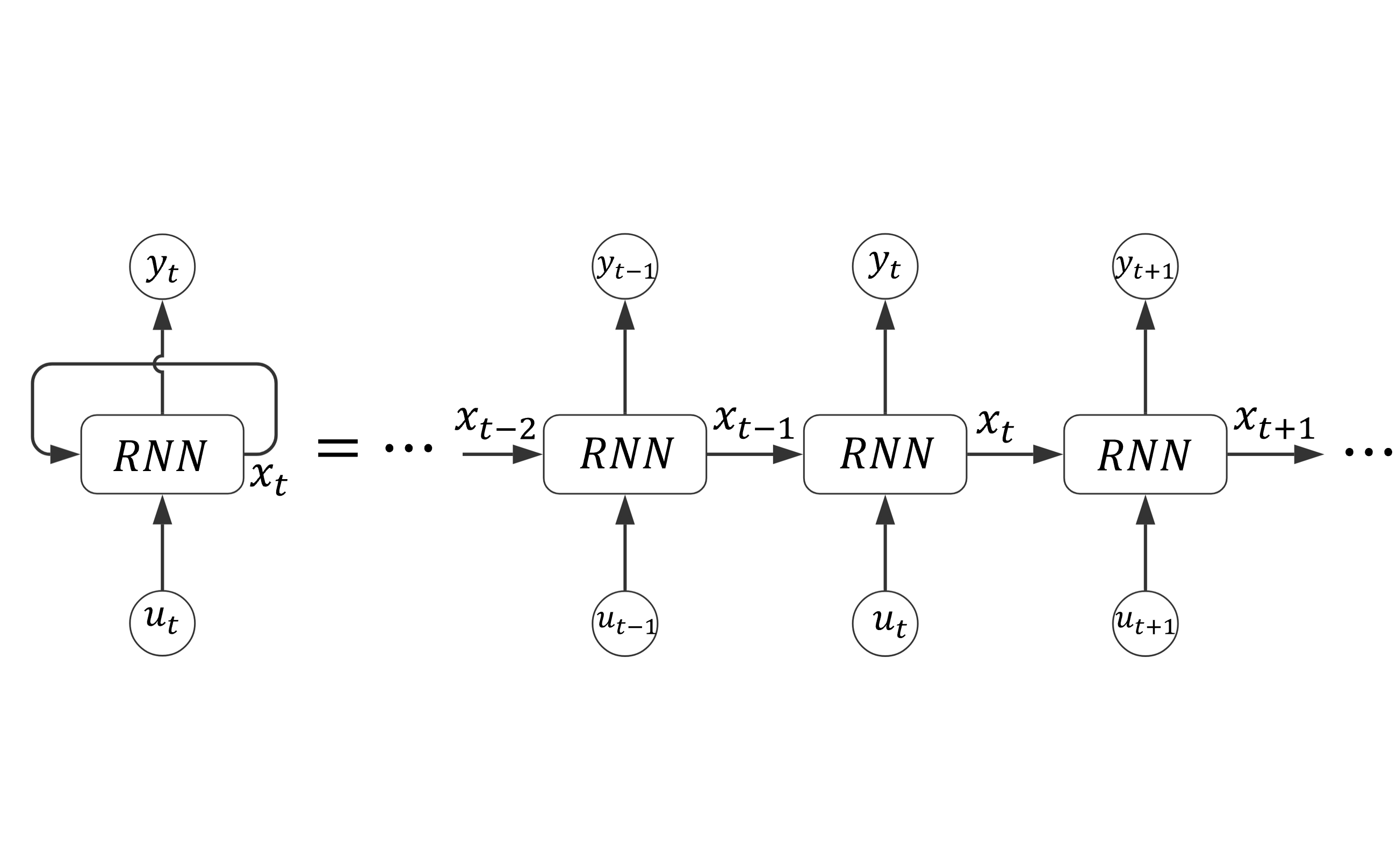}
    \caption{{Schematic diagram of a} recurrent neural network }
    \label{fig:RNN_unroll}
    \end{figure}

	Despite recent advancements in the area of artificial neural networks, most of the proposed architectures are fragile in presence of input noise. Due to the iterative nature of the recurrent models, the effect of input noise propagates and magnifies as the input sequence becomes longer. The noise from prior inputs that are passing through nonlinear functions make the analysis more challenging. Moreover, if the model is unstable, there is no guarantee that the model can reach a reasonable outcome under presence if exogenous noise.
    Some notions of robustness in artificial neural networks have been addressed recently \cite{weng2018evaluating} and in more detail for recurrent models under advisory attack by \cite{ko2019popqorn},\cite{papernot2016crafting}. 
	
    In classic control and estimation problems, it is assumed that for a given unknown model, the input can be controlled \cite{anderson2007optimal,zhou1996robust}. Therefore, one can stimulate the input signal and collect measurements from the output to obtain a reasonable estimation for the system \cite{kalman1960contributions},\cite{julier2004unscented}. During the supervised learning of a RNN model, however, our ability to control the input sequence is limited. A  series of data with their corresponding target outputs are usually provided beforehand. Even though solving the learning problem gives a reasonable estimation of the model parameters in terms of prediction accuracy, the question of how well the model will performs in presence of anonymous data, which is corrupted by noise, is yet to be explored. 
    
	Throughout this paper, we aim to elucidate the problem of robustness for general recurrent neural networks with additive input noise. We begin by defining an effective metric for the RNN model's robustness that can quantify the impact of input noise on the output and cost function. Next, by calculating  an upper bound for the robustness measure, we show that under what condition the robustness measure is bounded and explore venues related to training  stable RNN models and its effect on the robustness. We proceed to introduce a novel approach that provides an estimation for the output statistics inspired by ideas from the Extended Kalman Filter. Finally, we propose two different algorithms to enhance robustness of recurrent neural networks. Our extensive simulation results confirm that our proposed methodology significantly improves robustness of RNNs in presence of Gaussian input noise.

    \section{Problem Statement}
     The state-of-the-art machine learning technique to handle sequential data is the class of recurrent neural networks (RNN), which has been successfully employed for commercial applications like Google's voice \cite{hughes2013recurrent}. Circulating the incoming information in a loop provides us a model that keeps the memory of its past and current inputs. The dynamics of a general recurrent neural network can be modeled by   
	\begin{align} \label{RNN_dyn}
        \begin{split}
            \bx_t &= \F(\bx_{t-1},\bu_t,{\theta}_F),\\
            \by_t &= \G(\bx_{t},\theta_G)
        \end{split}
    \end{align}
	 with initial condition $\bx_0$, where $\bu_t \in \R^d,\bx_t \in {\R^{n}},\by_t \in {\R^m}$ represent the input, hidden state, and the output of the recurrent network at time instant $t$. The vectors of all trainable parameters are denoted by $\theta_F$ and $\theta_G$. Examples  include basic RNN\cite{pascanu2013difficulty}, LSTM \cite{hochreiter1997long}, and GRU \cite{cho2014learning}. Throughout the paper, $\norm{\bx}$  stands for the Euclidean norm of vector $\bx$.
	
	\vspace{0.1cm}
    \begin{asm}
    The vector-valued function $\F$ is Lipschitz continuous w.r.t  the hidden state and input with  constants  $\lambda >0$ and $\kappa_u >0$, respectively. Moreover function $\G$ is  Lipschitz continuous w.r.t the hidden state with constant $\kappa_G >0$. 
      \end{asm}
    \vspace{0.1cm}
    
    This assumption implies that 
    \begin{equation}\label{Lip-F}
        \norm{\,\F(\bx,\bu,\theta_F)-\F(\Bar{\bx},\bu,\theta_F)\,} \leq \lambda \norm{\,\bx-\Bar{\bx}\,}
    \end{equation}
    for all $\bx, \Bar{\bx} \in \R^n$ and  
    \begin{equation}
        \norm{\,\F(\bx,\bu,\theta_F)-\F(\bx,\Bar{\bu},\theta_F)\,} \leq \kappa_u \norm{\,\bu-\Bar{\bu}\,},
    \end{equation}
    for all $\bu,\Bar{\bu} \in \R^d $. 
  	The learning process in a RNN model involves finding \linebreak[4] parameters $\theta_F,\,\theta_G$ using given  training datasets $\left\{(\bu_{t},\by^*_{t})  \right\}_{t=0}^{T}$ by minimizing a loss function that measures closeness between the target output  $\left\{\by^*_{t}\right\}_{t=0}^{T}$ and the output of system \eqref{RNN_dyn} w.r.t input $\left\{\bu_{t}\right\}_{t=0}^{T}$. In this work, we consider the class of loss functions that can be expressed as    
    \begin{equation}\label{loss-fcn}
    \mathcal{E}(\theta_F,\theta_G)=\sum_{t=1}^T \mathcal{L}({\by}_{t},{\by}^*_{t}),
    \end{equation}
    where $\mathcal{L}$ is Lipschitz continuous with constant  $\kappa_{\mathcal{L}} > 0$.
    Examples of some commonly used loss functions are the cross-entropy function for classification purposes, which is defined by
    \[      \mathcal{L}(\,\by_t,\by_t^*\,) = -\by_t[{\bf i}^*]+\log \left(\sum_{j=1}^{m} \exp\left({\by_t\,[j]\,}\right) \right),\]
where   
   ${\bf i}^* = \arg\max_{j}~ \by_t^*\,[j]$ is a class label, and mean square error function to predict time series, which is given by   
    \[\mathcal{L} (\by_t,\by_t^*) = \norm{\by_t-\by_t^*}_2^2.\]

    The {\it problem} is to learn trainable parameters of the recurrent neural network \eqref{RNN_dyn} such that system \eqref{RNN_dyn} exhibits a robust behaviour in presence of additive input noise. Suppose that    
    the given training data  $\bu_t$ is corrupted by some additive Gaussian noise $\bw_t \sim \mathcal{N}(0,\Sigma_t)$, i.e., $\Tilde{\bu}_t=\bu_t + \bw_t$, and  fed to the RNN. The hidden state and output become random variables whose time evolution are given by  
    \begin{align}\label{RNN-Noise}
        \begin{split}
            \Tilde{\bx}_t &= \F(\Tilde{\bx}_{t-1},\Tilde{\bu}_t,\theta_F),\\
            \Tilde{\by}_t &= \G(\Tilde{\bx}_{t},\theta_G),
        \end{split}
    \end{align}
  The {\it objective} of this paper is to train network \eqref{RNN_dyn} using the given datasets $\left\{(\bu_{t},\by^*_{t})  \right\}_{t=0}^{T}$ such that the output of the perturbed system \eqref{RNN-Noise} minimizes the expected loss function 
  \begin{equation}\label{expected-loss}
  \tilde{\mathcal{E}}(\theta_F,\theta_G) = \mathbb{E}\left\{\sum_{t=1}^T\mathcal{L}(\Tilde{\by}_{t}, \by_t^*)\right\}.
  \end{equation}
  

    
    \begin{rem}
    Since the sequence of target outputs  $\left\{\by_t^*\right\}_{t=0}^T$ are given, and for simplicity of our notations, we may use notation $\mathcal{L}(\by_t)$ instead of $\mathcal{L}(\by_t,\by_t^*)$.
    \end{rem}

    \section{Robustness Measure for Learning}
We can interpret the disturbed RNN model \eqref{RNN-Noise} as a control system with noisy input and employ ideas from robust control \cite{zhou1996robust} to analyze this class of neural networks. We adopt the  expected deviation of the output of the noisy RNN \eqref{RNN-Noise} from the output of the undisturbed RNN \eqref{RNN_dyn} as a robustness measure, i.e.,   
    \begin{equation} \label{rbt-meas}
        \rho_t(\theta_F,\theta_G) := \E \left\{ \norm{\Tilde{\by}_t-\by_t}^2 \right\}.
    \end{equation}

By denoting the expected value of the  output by $\hat{\by}_t =\E [\Tilde{\by}_t]$  and its covariance  by
    \begin{align*}
        \bR_t=\E \left\{(\Tilde{\by}_t - \hat{\by}_t)(\Tilde{\by}_t - \hat{\by}_t)^T \right\},
    \end{align*} 
one can show that 
    \begin{equation}\label{trace-bias}
         \rho_t(\theta_F,\theta_G) = \mathrm{Tr}\left(\bR_{t}\right)+\mathrm{Bias}\left(\by_{t}\right),
    \end{equation}
 where the bias term is given  by
    \begin{equation*}
        \mathrm{Bias}(\by_{t}) = (\hat{\by}_{t}- {\by}_{t})^T(\hat{\by}_{t}- {\by}_{t}).
    \end{equation*}
   
Our objective is to learn parameters $\theta_F,\theta_G$ using the original dataset such that the RNN  will have a robust performance with respect to all noisy inputs $\left\{\bu_t + \bw_t\right\}_{t=0}^T$ with   $\bw_t \sim \mathcal{N}(0,\Sigma_t)$. Calculating the expected loss function \eqref{expected-loss} is very challenging in general as one should calculate statistics of the output of the noisy nonlinear system \eqref{RNN-Noise}, which requires solving the corresponding Fokker–Planck equation \cite{pavliotis2014stochastic},\cite{risken1996fokker}. Our following result quantifies a relationship that will help us solve the robust learning problem efficiently.

\begin{thm}\label{Thm-Upper-Bound}
The expected loss function for the noisy recurrent neural network \eqref{RNN-Noise} satisfies 
   \begin{align}\label{Var_Acc}
        \tilde{\mathcal{E}}(\theta_F,\theta_G)  \leq
         \kappa_{\mathcal{L}} \sum_{t=1}^T \sqrt{\rho_t(\theta_F,\theta_G)} + \mathcal{E}(\theta_F,\theta_G),
    \end{align}
    where $\kappa_{\mathcal{L}}$ is the Lipschitz constant of $\mathcal{L}$. 
\end{thm}
\begin{proof}
At any time instance $t$ one can write 
    \begin{align*}
        \mathcal{L}(\Tilde{\by}_{t})& =\mathcal{L}(\Tilde{\by}_{t})-\mathcal{L}({\by}_{t})+\mathcal{L}({\by}_{t}), \\
        & \leq \norm{\mathcal{L}(\Tilde{\by}_{t})-\mathcal{L}({\by}_{t,})} +\mathcal{L}({\by}_{t}),\\
        &\leq \kappa_{\mathcal{L}}\norm{\Tilde{\by}_{t} - {\by}_{t}}  +\mathcal{L}({\by}_{t}),
    \end{align*}
    taking expected value with respect to the input noise from both side implies that
    
    \begin{equation} \label{BV-ineq}
        \E\left\{\mathcal{L}(\Tilde{\by}_{t})\right\} \leq \mathcal{L}({\by}_{t}) +
         \kappa_{\mathcal{L}}\E\left\{\norm{\Tilde{\by}_{t}- {\by}_{t}}\right\}.
    \end{equation}
    
    For the second term by the Jensen's inequality and the fact that square root is a concave function 
    \begin{align*}
        \E\Big\{\norm{\Tilde{\by}_{t}- {\by}_{t}}\Big\} &= \E\left\{ \sqrt{(\Tilde{\by}_{t}- {\by}_{t})^T(\Tilde{\by}_{t}- {\by}_{t})}\right\}\\
        & \leq \sqrt{ \E \left\{ \norm{\Tilde{\by}_t-\by_t}^2 \right\}}\\
        & = \sqrt{ \rho_t(\theta_F,\theta_G)}.
    \end{align*}
    The above inequality combined with inequality \eqref{BV-ineq} implies that
    \begin{equation*}
        \E\left\{\mathcal{L}(\Tilde{\by}_{t})\right\} \leq \mathcal{L}({\by}_{t}) +
         \kappa_{\mathcal{L}}\sqrt{ \rho_t(\theta_F,\theta_G)}.
    \end{equation*}
    Taking the summation over time from both sides of the inequality result in 
    \begin{equation*}
        \sum_{t=1}^T \E\left\{\mathcal{L}(\Tilde{\by}_{t})\right\} \leq  \sum_{t=1}^T \mathcal{L}({\by}_{t}) +
         \sum_{t=1}^T \kappa_{\mathcal{L}}\sqrt{ \rho_t(\theta_F,\theta_G)}.
    \end{equation*}
    The expected value and summation are exchangeable. Therefore, from the definition of \eqref{expected-loss}, the inequality \eqref{Var_Acc} follows.
\end{proof}

\vspace{0.1cm}

    
    The right-hand-side of the inequality \eqref{BV-ineq} provides a meaningful   \textit{decomposition} that involves the learning accuracy, which can be quantified by the loss function \eqref{loss-fcn}, and the impact of noise on learning, which is quantified by the square root of the robustness measure. This result is important as it suggests that in order to achieve robust learning with respect to exogenous noise, one should train the RNN by minimizing a regularized cost function that represents a term for learning accuracy and an additional term for robustness. We will discuss this in detail in  Section \ref{sec:Learning Robust RNN}.

    
\begin{rem}
The robustness measure \eqref{rbt-meas} has been already used in evaluating the $\mathcal{H}_2$-norm of linear time-invariant systems that are subject to Gaussian noise \cite{doyle1988state,bamieh2003exact, siami2017new,siami2017growing,siami2017centrality,mousavi2020explicit,mousavi2020koopman} 
\end{rem}

\section{Calculating Upper Bounds for the Robustness Measure}
As we discussed earlier, finding explicit forms for the robustness measure \eqref{rbt-meas} is a tedious task as one needs to calculate statistics of a stochastic process generated by a nonlinear dynamical system. In the following, we aim at formulating some explicit upper bounds under some assumptions.     



\vspace{0.1cm}
 \begin{defn} 
The RNN model \eqref{RNN_dyn} is stable if there exists $\theta_F$ such that the Lipschitz property \eqref{Lip-F} holds with $\lambda < 1$. 
\end{defn}
\vspace{0.1cm}

We refer to \cite{miller2018stable} for a complete discussion on this class of networks. 
Stable recurrent neural networks are more reliable due to their predictable behavior. In the following theorem, we show that if an RNN  is stable, then the robustness measure is bounded for every data sequence.

\begin{thm}\label{thm2}
Let us consider the noiseless RNN \eqref{RNN_dyn} with initial condition $\bx_0$, and the noisy RNN \eqref{RNN-Noise} with additive Gaussian noise $\bw_t \sim \mathcal{N}(0,\Sigma_t)$ and initial condition that is drawn from $\mathcal{N}(\bx_0, \Gamma)$. Then, 
\begin{align}\label{UP1}
    \rho_t(\theta_F,\theta_G)  \leq \kappa_G^2\left((2\lambda^2)^t\mathrm{Tr}(\Gamma)  + \kappa_u \sum_{i=0}^{t-1} (2\lambda^2)^{i}   \mathrm{Tr} (\Sigma_i)\right).
\end{align}
Moreover, if $\lambda < \frac{1}{\sqrt{2}}$ and the covariance matrix of the input noise stays  constant over time, i.e.,  $\Sigma_t=\Sigma$ for all $t \geq 0$, then   
\begin{align}\label{UP2}
     \rho_{\infty}(\theta_F,\theta_G)  \leq \frac{2(\kappa_u \kappa_G)^2 \mathrm{Tr} (\Sigma)}{1-2\lambda^2}.
\end{align}
\end{thm}

\vspace{0.2cm}
\begin{proof}
At each time instant $t$, we have
    \begin{align*}
    \norm{\bx_t-\Tilde{\bx}_t} &= \norm{\F(\bx_{t-1},\bu_t-\F(\Tilde{\bx}_{t-1},\Tilde{\bu}_t)} \\
    &=||\F(\bx_{t-1},\bu_t)-\F(\bx_{t-1},\Tilde{\bu}_t)\\
    &\hspace{2cm} + \F(\bx_{t-1},\Tilde{\bu}_t) -\F(\Tilde{\bx}_{t-1},\Tilde{\bu}_t)||\\
    &\leq \norm{\F(\bx_{t-1},\bu_t)-\F(\bx_{t-1},\Tilde{\bu}_t)} \\ 
    &\hspace{2cm} + \norm{\F(\bx_{t-1},\Tilde{\bu}_t) -\F(\Tilde{\bx}_{t-1},\Tilde{\bu}_t)}\\
    &\leq \kappa_u \norm{\Tilde{\bu}_t-\bu_t} + \lambda \norm{\Tilde{\bx}_t-\bx_t} \\
    &\leq \kappa_u \norm{\bw_t} + \lambda \norm{\Tilde{\bx}_{t-1}-\bx_{t-1}}.
\end{align*}
Squaring both sides results in
\begin{align}
    \norm{\Tilde{\bx}_{t}-\bx_{t}}^2 &\leq \big(\kappa_u \norm{\bw_t} + \lambda \norm{\Tilde{\bx}_{t-1}-\bx_{t-1}} \big)^2 \nonumber\\ \label{ineq-hidden}
    & \leq 2 \big( \kappa_u^2 \norm{\bw_t}^2  + \lambda^2\norm{\Tilde{\bx}_{t-1}-\bx_{t-1}}^2\big),
\end{align}
where inequality  $(x+y)^2\leq 2(x^2+y^2)$,  for all $x,y \geq 0$, is utilized in the latter step. 
The Lipschitz continuity of $\G$ implies that
\begin{align*}
    \norm{\Tilde{\by}_{t}-\by_{t}} =\norm{\G(\bx_t)-\G(\Tilde{\bx_t}))} \leq \kappa_G\norm{\Tilde{\bx}_{t}-\bx_{t}}.
\end{align*}
The robustness measure can be bounded by
\begin{equation}
    \rho_t(\theta_F,\theta_G)  = \E \left\{\norm{\Tilde{\by}_{t}-\by_{t}}^2\right\}  \nonumber  \leq \kappa_G^2 \E \left\{\norm{\Tilde{\bx}_{t}-\bx_{t}}^2\right\}. \label{rho-hiden}
\end{equation}
Since $\bw_t$ is Gaussian with zero mean and  covariance $\Sigma_t$, it can be shown that 
$ \E \{ \norm{\bw_t}^2\} = \mathrm{Tr} (\Sigma_t)$. Thus, inequality \eqref{ineq-hidden} can be rewritten as 
\begin{equation*}
    \E\left\{\norm{\Tilde{\bx}_{t}-\bx_{t}}^2\right\} \leq  2\kappa_u^2 \mathrm{Tr} ( \Sigma_t ) +2 \lambda^2  \E \left\{ \norm{\Tilde{\bx}_{t-1}-\bx_{t-1}}^2\right\}.
\end{equation*}
By inserting the previous error bounds recursively, one can obtain an explicit upper bound as follows
\begin{align}
    \E\left\{\norm{\Tilde{\bx}_{t}-\bx_{t}}^2\right\} & \leq 
     (2 \lambda^2)^t \E \left\{ \norm{\Tilde{\bx}_{0}-\bx_{0}}^2\right\} \nonumber \\ & + 2\sum_{i=0}^{t-1} (2 \lambda^2)^i \kappa_u^2 \mathrm{Tr} (\Sigma_i). \label{rho-hiden}
\end{align}
This inequality combined with the fact that $\E \left\{ \norm{\Tilde{\bx}_{0}-\bx_{0}}^2\right\}=\mathrm{Tr}(\Gamma)$ leads to 
\begin{align*}
    \rho_t(\theta_F,\theta_G)  \leq \kappa_G^2\left((2 \lambda^2)^{t}\mathrm{Tr}(\Gamma)  + 2 \sum_{i=0}^{t-1} (2 \lambda^2)^{i}  \kappa_u^2 \mathrm{Tr} (\Sigma_i)\right).
\end{align*}
For stable recurrent networks, we have $\lambda < \frac{1}{\sqrt{2}}$. Thus,  $\lambda^t \rightarrow 0 $ as $t \rightarrow \infty$ and the geometric series  converges to
\begin{align*}
    \rho_{\infty}(\theta_F,\theta_G)  \leq 2\frac{(\kappa_u \kappa_G)^2 \mathrm{Tr} (\Sigma)}{1-2 \lambda^2}.
\end{align*}
\end{proof}

Despite our usual expectations, aiming at training stable RNN models leads to imposing unnecessary constraints on the trainable parameters (cf. \cite{miller2018stable}), which may result in poor robustness properties. Moreover, stability is not a requirement when training recurrent neural networks as they usually operate over a finite time horizon. In Section \ref{sec:Learning Robust RNN}, we discuss that training without imposing stability conditions will result in superior robustness properties.

The dynamics of a basic recurrent neural network model is governed by
\begin{align}\label{Basic-RNN}
\begin{split}
        \bx_t &= \sigma(\A \bx_{t-1}+\B \bu_t +b),\\
        \by_t &= \C \bx_{t}+ c,
\end{split}
\end{align}
where $\sigma$ is Lipschitz continuous with constant $\kappa_\sigma$. The trainable parameters are components of matrices $\A, \B, b, \C$, and $c$. Some popular  examples of $\sigma$ are $\tanh$, and Relu function, where in both cases $\kappa_{\sigma} = 1$. The corresponding Lipschitz constants are
\[  \lambda=\kappa_\sigma \norm{\A},~
    \kappa_u =\kappa_\sigma \norm{\B},~
    \kappa_G =\norm{\C},\]
where the matrix norm is defined by 
\begin{equation*}
    \norm{A}=\sup \left\{\norm{Ax}~|~x\in \R^n \SP \text{with} \SP \norm{x}=1 \right\}.
\end{equation*}

\vspace{0.1cm}
\begin{cor}
For the class of basic RNNs, if we assume that the covariance matrix of the input noise stays constant over time and equal to $\Sigma$ and $\kappa_{\sigma}$ is set to $1$ then upper bounds \eqref{UP1}-\eqref{UP2} can be improved by
\begin{align}\label{UP1-BR}
    \rho_t (\theta_F,\theta_G) \leq \norm{\C}^2 \Big(\norm{\A}^{2t} \mathrm{Tr}(\Gamma) \\+ \sum_{i=0}^{t-1} \norm{\A}^{2i} \norm{\B}^2 \mathrm{Tr} (\Sigma) \Big)\nonumber,
\end{align}
and
\begin{equation}\label{UP2-BR}
    \rho_\infty (\theta_F,\theta_G) \leq   \left(\frac{\norm{\B} \norm{\C}}{\sqrt{1-\norm{\A}^2}} \right)^2 \mathrm{Tr} (\Sigma).
\end{equation}
\end{cor}
\begin{proof}
Let us define $h_t = \A\bx_{t-1} +\B\bu_t +b$ for ease of notation, then by Lipschtiz continuity of $\sigma$ 

\begin{align*}
    \norm{\tilde{\bx}_t - \bx_t} &\leq \norm{\sigma(\tilde{h}_t) - \sigma(h_t)}\\
    &\leq \kappa_\sigma \norm{\tilde{h}_t-h_t} \\
    & = \kappa_\sigma \norm{\A(\tilde{\bx}_{t-1}-\bx_{t-1})+B\bw_t}.
\end{align*}
By using the above inequality it is clear that
\begin{align*}
    &\E \left\{ \norm{\delta_t}^2\right\} \leq   \E \left\{\kappa_\sigma^2\norm{\A(\delta_{t-1})+B\bw_t}^2\right\} \\
    &\leq \kappa_{\sigma}^2\E \left\{ \delta_{t-1}^T\A^T \A\delta_{t-1} + \bw_t^T\B^T\B\bw +2 \delta_{t-1}^T\A^T \B \bw_t\ \right\},
\end{align*}
where $\delta_t =\tilde{\bx}_t - \bx_t$. Note that the input noise is independent through time i.e. $\E\left\{\bw_t^T\bw_s\right\} = 0$ if $t\neq s$. Hence the $\tilde{x}_{t-1}$ is independent from $\bw_t$ and   $\E\left\{\bw_t^T\delta_{t-1}\right\} = 0$. Therefore the above inequality boils down to
\begin{align}
     \E \left\{ \norm{\delta_t}^2\right\} &\leq\kappa_{\sigma}^2\E \left\{ \delta_{t-1}^T\A^T \A\delta_{t-1} + \bw_t^T\B^T\B\bw \right\} \nonumber   \\
    & \hspace{-5 mm}  \leq \kappa_\sigma^2 \Big( \norm{\A}^2 \E\left\{ \norm{\delta_{t-1}}^2 \right\} + \norm{\B}^2 \mathrm{Tr}(\Sigma_t)  \Big) \label{linear-app}.
\end{align}
Replacing the inequality \eqref{ineq-hidden} by \eqref{linear-app} with similar approach to proof of Theorem \ref{Thm-Upper-Bound} the inequalities \eqref{UP1-BR}, \eqref{UP2-BR}  follows.


\end{proof}
\vspace{0.1cm}
 
We should point out that the upper bounds obtained in this section are based on Lipschitz constants. Therefore, they are not expected to be tight.

	\begin{figure}[t]
        \centering
        \includegraphics[width=0.4\textwidth,trim=0 0 0 20,clip]{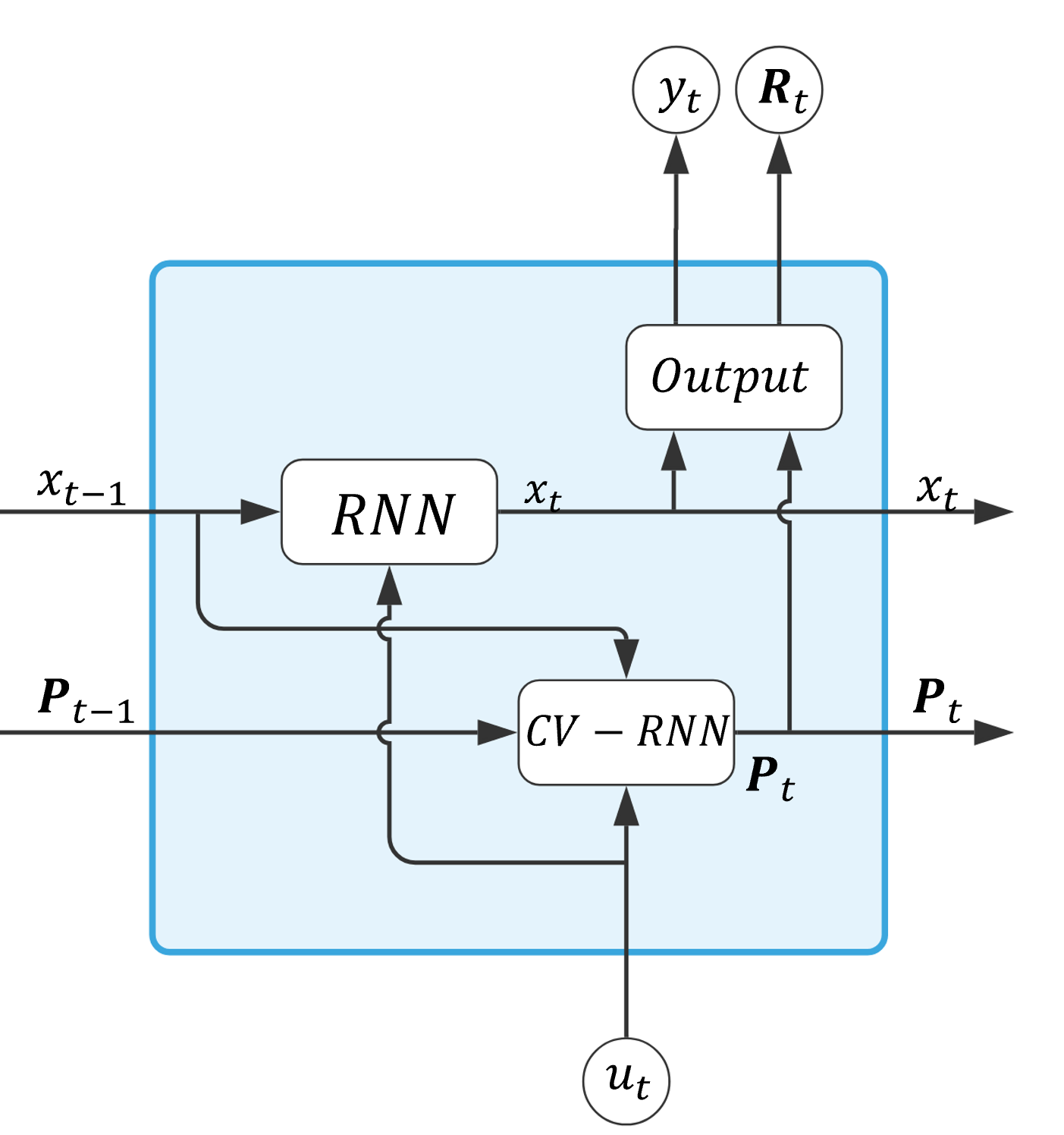}
        \caption{Robust learning architecture.}
        \label{fig:Robustlearner}
    \end{figure}
    
	\section{State and Output Covariance Estimation}
    Computing an exact and explicit expression for the robustness measure requires us to access the exact stochastic information of the disturbed RNN output, which is tremendously challenging, even for the basic RNN models.
    The problem of estimating the statistics of the output of a noisy nonlinear control system has been studied comprehensively in the context of Kalman filtering  \cite{kalman1960contributions}-\cite{julier2004unscented},\cite{smith1962application}. The Extended Kalman Filter (EKF) and its variants are powerful methods that have been successfully applied to various real-world applications  \cite{wan1997neural}. In this section, we build upon ideas from the EKF and apply similar linearization techniques to estimate covariance matrices of the noisy RNN. In the following, we briefly state some of the known results on the transformation of uncertainty \cite{julier1997new}. 

Suppose that a function $f : \R^n \rightarrow \R^m$ in $\mathcal{C}^1$ and a random vector $\bx \in \R^n$  with expected value $\hat{\bx}$ and covariance matrix ${\bf P}_{xx}$ are given. Then, one can expand right-hand-side of \linebreak[4] ${\by}= f(\bx)$ around $\hat{\bx} = \E\left\{\bx\right\}$ to obtain 
    \begin{equation}
        \by = f(\hat{\bx})+ \nabla_x f(\hat{\bx}) (\bx-\hat{\bx}) + {o}\big( (\bx-\hat{\bx})^2 \big).
    \end{equation}
    Taking the expectation from both sides of the equality results in 
    \begin{align}
        \hat{\by} &= f(\hat{\bx})+ \nabla_x f(\hat{\bx}) \E\left\{\bx-\hat{\bx}\right\} + o \big(\E
        \left\{(\bx-\hat{\bx})^2\right\} \big) \nonumber \\
        &= f(\hat{\bx})+\E \left\{ o \big( (\bx-\hat{\bx})^2 \big) \right\} \nonumber\\
        &\approx f(\hat{\bx}),  \label{bias-estimator}
    \end{align}
    where $\E[\by]=\hat{\by}$. Furthermore, the covariance matrix of the random variable $\by$ can be approximated by 
    \begin{equation}\label{Cov-estimator}
        {\bf P}_{yy} \approx \left(\nabla_x f \right) \ {\bf P}_{xx} \left(\nabla_x f\right)^T.
    \end{equation}
The Extended Kalman filter utilizes a similar approach to approximate expected values and covariance matrices. Our main advantage in this problem is that we have full access to the exact value of state and output variables with a predetermined input sequence.

    \begin{figure}[t!]
        \centering
        \includegraphics[width=.5\textwidth,trim=20 0 0 0 ,clip]{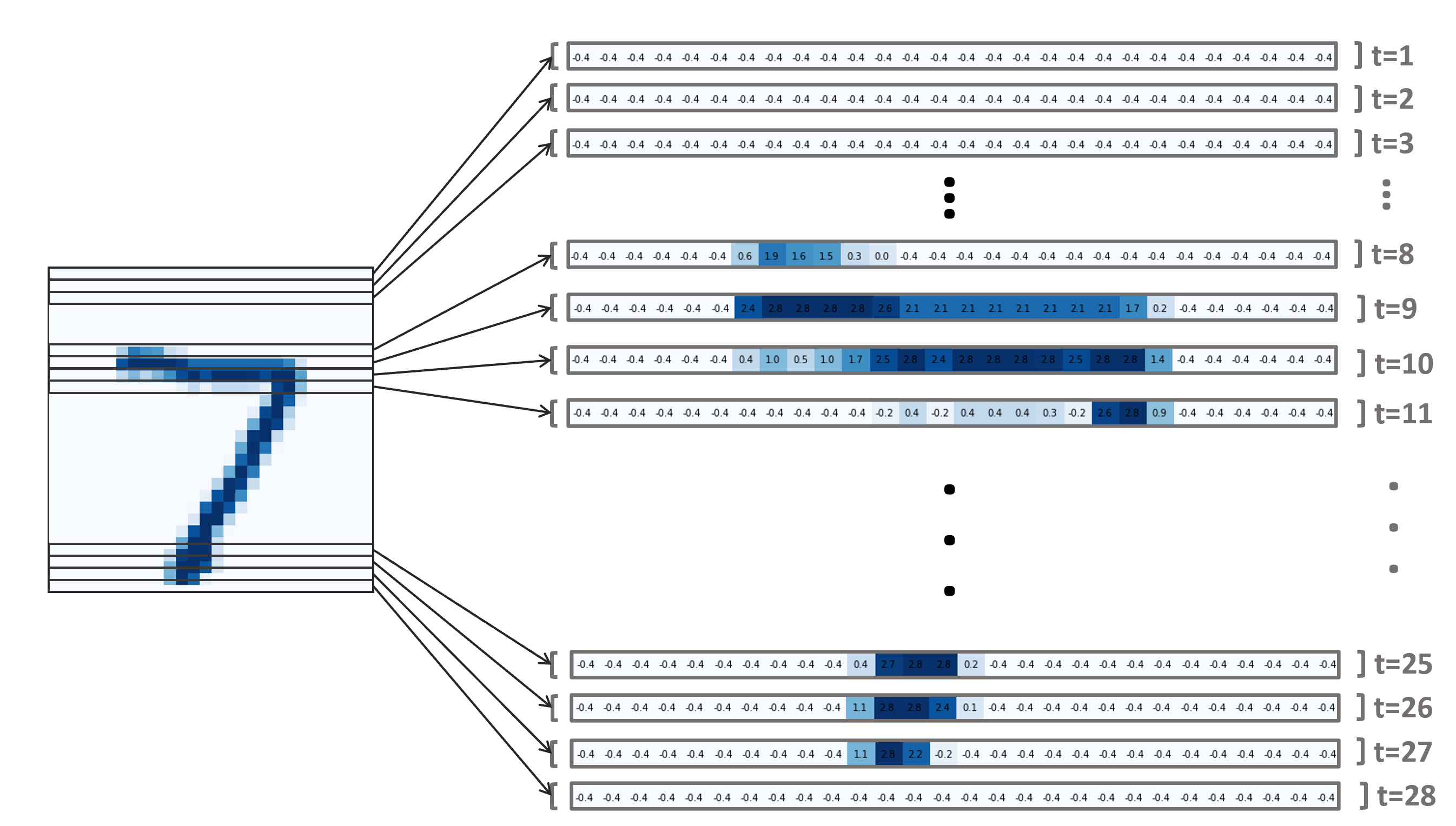}
        \caption{ The sequencing of the MNIST dataset for recurrent models, illustrated with digit $7$.}
        \label{fig:dataprep}
    \end{figure}
    
    \begin{thm}\label{Thm-Trace-R}
    Let us consider the noisy recurrent neural network \eqref{RNN-Noise} with input noise $\bw_t \sim \mathcal{N}(0,\Sigma_t)$. Then, the estimation for the expected value and covariance matrix of the hidden state can be calculated by
    \begin{align}
        \begin{split}\label{RNN_Bias}
             \hat{\bx}_{t} &=  \F(\hat{\bx}_{t-1},{\bu}_t,\theta_F),\\
        \hat{\bf \bP}_t &= (\nabla_x \F)\, \hat{\bf \bP}_{t-1} \, (\nabla_x \F)^T + (\nabla_u \F) \, \Sigma_t \, (\nabla_u \F)^T,
        \end{split}
    \end{align}
    where $\hat{\bx}_t =\E [\Tilde{\bx}_t]$ and $\hat{\bP}_t$ denotes the estimation of the covariance matrix $\bP_t=\E \left\{(\Tilde{\bx}_t - \hat{\bx}_t)(\Tilde{\bx}_t - \hat{\bx}_t)^T \right\}$. Similarly, the expected value and covariance matrix of the output can be estimated as
    \begin{align}
        \begin{split}\label{RNN_outBias}
        \hat{\by}_t &=  
      \G(\hat{\bx}_{t},\theta_G),\\
        \hat{\bf \bR}_t  &=  
      (\nabla_x \G) \hat{\bf \bP}_{t} (\nabla_x \G)^T,
        \end{split}
    \end{align}
    where $ \hat{\bf R}_t$ is the estimation of covariance of the output. The gradients  $\nabla_x \F,\, \nabla_u \F$, and $\nabla_x \G$ are calculated at working points $(\hat{\bx}_{t-1},\,{\bu}_t,\,\theta_F)$ and $(\hat{\bx}_{t-1},\, \theta_G)$ respectively. 
    Furthermore, the robustness measure can be approximated by 
    \begin{equation}\label{perf-rt}
        \rho_{t}(\theta_f,\theta_G)  \approx \mathrm{Tr}\big({\hat{\bf R}}_t\big).
    \end{equation}
    \end{thm}
    \vspace{0.1cm}
    \begin{proof}
    Based on equation \eqref{RNN_dyn}, the linearization around $ \hat{\bx}_{t-1}$ for the dynamics of the hidden state implies that
    \begin{align*}
        \Tilde{\bx}_t &= \F(\Tilde{\bx}_{t-1},\Tilde{\bu}_t)\\
        & = \F(\hat{\bx}_{t-1} + \delta\bx_{t-1} ,\hat{\bu}_t + \bw_t)\\
        & = \F(\hat{\bx}_{t-1} ,\hat{\bu}_t) + \nabla_x \F \delta\bx_{t-1} + \nabla_u \F \bw_t \\
        & \hspace{3cm} + o(\delta \, \bx_{t-1}^2 ,\bw_t^2).
    \end{align*}
    
    Similar to transformation of uncertainty technique, one can employ equations \eqref{Cov-estimator} and \eqref{bias-estimator} to  estimate the expected value and covariance of $\Tilde{\bx}_t$ by equation \eqref{RNN_Bias}. Exogenous noise
     $\bw_t$ has  Gaussian distribution $\mathcal{N}(0,\Sigma_t)$ with assumption that the additive noise is independent throughout time \linebreak[4] i.e. $\E[\bw_t \bw_s]=\delta_{t-s}$.  Likewise, the linearization  for the output is going to be
    \begin{align*}
         \Tilde{\by}_t &= \G(\Tilde{\bx}_{t}) = \G(\hat{\bx}_{t} +  \delta\bx_{t} )\\
         &=\G(\hat{\bx}_{t}) + \nabla_x \G  \delta\bx_{t}  + {o}( \delta\bx_{t}^2). \
    \end{align*}
    
    Employing transformation of uncertainty, it is straightforward to compute the output expected value and covariance by equations \eqref{RNN_outBias}.

    In each training epoch, our access to the $\hat{\by_t}$ is limited to the sampling of $\by_t$. If the original data is clean, i.e., there are no disturbance or corruption applied to the input sequence, then $\hat{\by}_t=\by_t$. However, in most cases, the original data is not clean, and there is some error. There are various methods to improve this approximation that is out of the scope of this paper \cite{julier1997new},\cite{julier2004unscented}. Therefore, we assume that the bias term in the robustness measure is zero, i.e., $\hat{\by}_t \approx \by_t$.
    Finally, the robustness measure can be approximated by 
    \begin{align*} 
        \rho_{t}(\theta_F,\theta_G)  = \mathrm{Tr}\big({\bf R}_t\big) \approx  \mathrm{Tr}\big(\hat{\bf R}_t\big).
    \end{align*}
    \end{proof}
 
 The iterative update rules \eqref{RNN_Bias} and \eqref{RNN_outBias} along the evolution of the RNN provide reliable estimates for the statistics of the desired variables which in turn enable us to estimate the robustness measure in an efficient manner. Although this approach is based on approximating a nonlinear system with its linearized counterpart, it turns out that it provides reliable and efficient estimates for our learning purposes.   
 

    \begin{exmp}
    Let us consider the basic RNN whose dynamics is governed by \eqref{Basic-RNN}. 
    For simplicity of our notation, we use $h_t = \A\bx_{t-1} +\B\bu_t +b$. The gradient of the recurrent model with respect to various variables are given by
    \begin{align*}
        \nabla_x \F(h_t) &=\sigma' (h_t) \odot \A, \\
        \nabla_u \F(h_t) &= \sigma' (h_t)\odot \B,\\
        \nabla_x \G(x_t) &= \C, 
    \end{align*}
    where $\odot$ represents the Hadamard product and $\sigma'(h_t)$ stands for  the element-wise derivative of the $\sigma$ function at $h_t$.  The update rules for matrices $\hat{\bf P}_t$ and $\hat{\bf R}_t$ are
    \begin{align}
    \begin{split}\label{Basiv-var1}
        \hat{\bP}_t = \big(\sigma' (h_t) \odot \A \big) \bP_{t-1} \big(\sigma' (h_t) \odot \A \big)^T + \\
        \big(\sigma' (h_t) \odot \B \big) \Sigma_t \big(\sigma' (h_t) \odot \B \big)^T, 
    \end{split}
    \end{align}
    \begin{align}\label{Basiv-var2}
        \hat{\bR}_t \approx \C \bP_t \C^T.
    \end{align}
    \end{exmp}
  \vspace{0.1cm}
  
 The iterative equations \eqref{Basiv-var1}-\eqref{Basiv-var2} reveal how the trainable parameters influence the RNN's outcomes when they are subject to exogenous noise. We should highlight that the analysis in this section shows how we can efficiently approximate the robustness measure and employ it for training a robust recurrent neural network. 
    
        
    
    \begin{figure}[t]
    \begin{subfigure}[t]{.49\linewidth}
        \centering
    	\includegraphics[width=\linewidth]{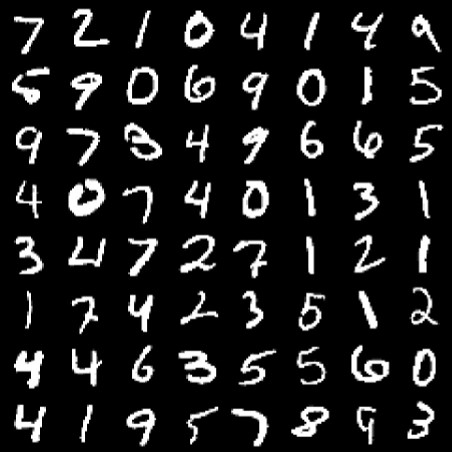}
    	\caption{$\omega= 0$.}
    	\label{fig:risk_collide_50}
    \end{subfigure}
    \hfill
    \begin{subfigure}[t]{.49\linewidth}
        \centering
    	\includegraphics[width=\linewidth]{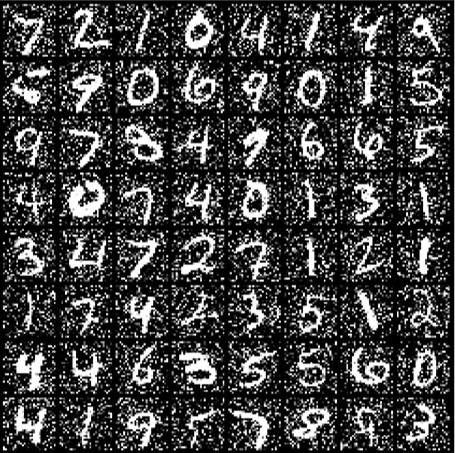}
    	\caption{$\omega= 1$.}
    	\label{fig:risk_collide_path_10}
    \end{subfigure}
    \medskip
    \begin{subfigure}[t]{.49\linewidth}
        \centering
    	\includegraphics[width=\linewidth]{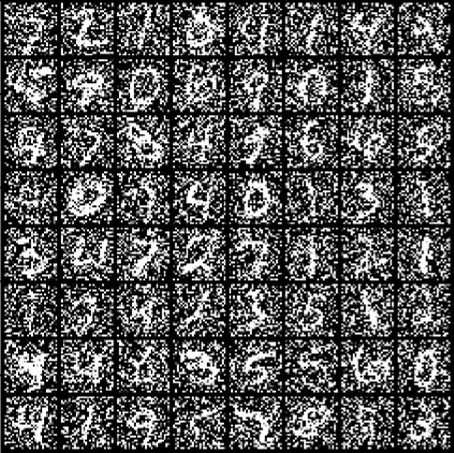}
    	\caption{$\omega= 2$.}
    \end{subfigure}
    \hfill
    \begin{subfigure}[t]{.49\linewidth}
        \centering
    	\includegraphics[width=\linewidth]{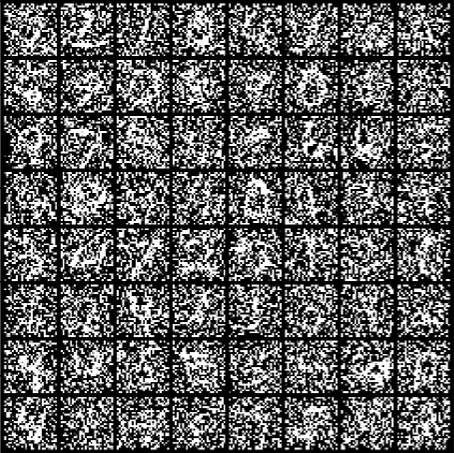}
    	\caption{$\omega= 3$.}
    \end{subfigure}
    \caption{Distorted MNIST dataset with different noise amplitudes.}
    \label{fig:Noisy data}
\end{figure}

    \section{ Learning Robust RNN} \label{sec:Learning Robust RNN}

    In previous sections, we quantified the effect of noise input on the output of a RNN network. As we discussed in Section III, the expected value of the training cost in presence of noise is bounded by    
     \begin{align*}
        \E\left\{ \mathcal{L}(\Tilde{\by}_{t}) \right\} \leq
         \kappa_{\mathcal{L}}\sqrt{ \rho_t(\theta_F,\theta_G)} +\mathcal{L}({\by}_{t}).
    \end{align*}
    This inequality reveals that including the robustness measure $\rho_t(\theta_F,\theta_G)$ as a  regularization term to the loss function, i.e.,
    \begin{equation}\label{reg-loss}
       \hat{\mathcal{E}}_t = \mathcal{L}(\by_t) + \mu \rho_t(\theta_F,\theta_G),
    \end{equation}
    where $\mu > 0$ is a design parameter to adjust robustness level, will significantly improve learning robustness. In view of \eqref{reg-loss} and the fact that calculating an explicit closed-form expression for the robustness measure is very challenging, we consider two remedies. The first approach is based on the upper bound in the right hand-side of inequality \eqref{UP1}, which is obtained using Lipschitz properties of the RNN. The second method is based on approximation formula \eqref{perf-rt}, which is derived by calculating propagation of noise in the output of the RNN throughout time.



    \subsection{Regularization via Robustness Measure Approximation}\label{Subsec-VI-1}
    
    The existing methods \cite{pascanu2013difficulty} in the context of recurrent neural networks only consider the bias term in \eqref{trace-bias} during the training and ignore the convariance term. We utilize an estimation of the output covariance in order to approximate  the robustness measure and regularize the loss function according to 
    \begin{equation}\label{reg-rbst-msr}
        \hat{\mathcal{E}}=\sum_{t=0}^T \bigg(\,\mathcal{L}(\by_t) + \mu \mathrm{Tr}(\hat{ \bR}_t) \,\bigg).
    \end{equation}
    
    The equations  \eqref{RNN_Bias}-\eqref{RNN_outBias} present a procedure to approximate the covariance matrices of the hidden states and the output variables. A schematic diagram of this process is shown in Figure \ref{fig:Robustlearner}, where the RNN block updates the hidden state biases and the CV-RNN block updates the covariance matrix of the hidden state.

    
    One may employ a variety of  backpropagation through time algorithms \cite{martens2011learning}-\cite{bengio2013advances} to compute the respective gradients for both the loss function and the surrogate for the robustness measure to learn a robust RNN.
    


    \subsection{Regularization via Upper Bound for  Robustness Measure}\label{Subsec-VI-2} 
   
   
   In this approach, we aim to minimize 
  the upper bound obtained in Theorem \ref{thm2} for the robustness measure.  As a result, the regularized loss function for training\footnote{{Further discussions on the decay rate of  factor $\frac{\mu}{N_e}$ can be found in \cite{pascanu2013difficulty}.}} can be expressed as  to the cost function such that
    \begin{equation}\label{Loss-Omega}
       \hat{\mathcal{E}}=\sum_{t=0}^T \bigg(\, \mathcal{L}(\by_t) +  \frac{\mu}{{N_e}}\Omega_t \,\bigg),
    \end{equation}
    where $N_e$ measures the number of epochs, $\mu$ is a design parameter, and $\Omega_t$ is the upper bound in inequality \eqref{UP1} that is given by
    \begin{equation*}
    \Omega_t =  \kappa_G^2\left((2\lambda^2)^t\mathrm{Tr}(\Gamma)  + \kappa_u \sum_{i=0}^{t-1} (2\lambda^2)^{i}   \mathrm{Tr} (\Sigma_i)\right).
    \end{equation*}
    For a basic recurrent model \eqref{Basic-RNN}, the upper bound in   inequality \eqref{UP1-BR} can be used
    \begin{equation*}
    \Omega_t = \norm{\C}^2 \left(\norm{\A}^{2t} \mathrm{Tr}(\Gamma) + \\ \sum_{i=0}^{t-1} \norm{\A}^{2i} \norm{\B}^2 \mathrm{Tr} (\Sigma_i) \right).
    \end{equation*}

    Imposing the upper bound as a regularizer steers the training process to learn a more robust system with respect to the input noise. Under some conditions, the regularizer indirectly enforces a  basic RNN to be stable, i.e., the $\lambda < 1$. In fact, for basic RNN, if after training one can  verify that 
    \[\Omega_t < \norm{\C}^2\left( \mathrm{Tr} (\Gamma)+t\norm{\B}^2 \mathrm{Tr} (\Sigma)\right), \]
    then, the resulting RNN will be  asymptotically stable.

    \begin{rem} The regularized problem in Subsection \ref{Subsec-VI-1} results in a comparably more robust RNN model with respect to what the regularized problem in Subsection \ref{Subsec-VI-2} can provide. However, the  computational cost of computing covariance matrices in Subsection \ref{Subsec-VI-1} is significantly higher. \end{rem}
        
    \begin{figure}[t]
        \centering
        \includegraphics[width=0.41\textwidth]{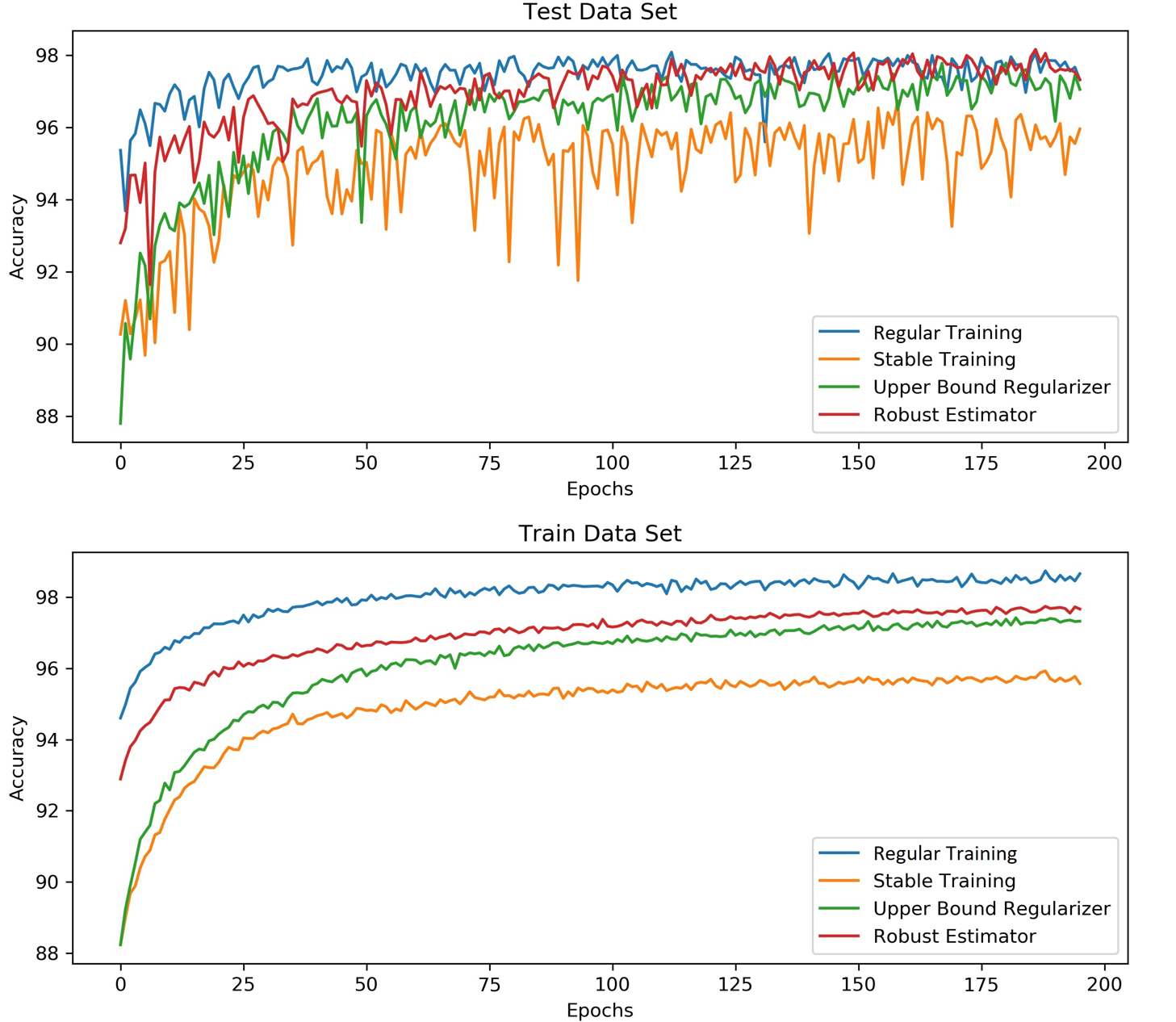}
        \caption{Learning curves for both training and testing datasets{, with aforementioned learning methods.}}
        \label{fig:Learning curve}
    \end{figure}

\begin{rem}
 The idea of training for a stable RNN may not necessarily improve robustness of RNN networks. One can observe from our derivations that the robustness measure depends on both weight matrices $\A$ and $\B$, whereas stability requirement only depends on $\A$. To enforce stability condition, the singular values of the weight matrix $\A$ are trimmed down to $1$ after each iteration  \cite{miller2018stable}. Enforcing stability conditions can potentially result in gradient vanishing problem, which will in turn deteriorate the learning accuracy. 

\end{rem}


\section{Experiments}

To validate our theoretical results presented in the previous  sections, we  consider the benchmark classification problem of handwritten digits using MNIST dataset \cite{lecun1998gradient}. Each sample in this dataset is a $28 \times 28 $ pixels image of a handwritten digit, which  can be represented by a $28 \times 28$ matrix, and is labeled by a target  digit between $0$ and $9$.  As it is shown in Figure \ref{fig:dataprep}, each sample is turned into a sequence of $28$ row vectors in order to make the data suitable to be utilized as input for recurrent neural networks. 

All simulations are performed in Pytorch platform\footnote{All codes for the experiments are available on first author's GitHub at https://github.com/ara416/Robust-RNN-Learning}. For the MNIST classification problem using a basic recurrent model, we investigate and compare four scenarios: regular  learning, stable learning, regularized robust learning using estimates, and regularized robust learning using an upper bound. 

All hyper-parameters that are common in between these models are set to be the same. For these experiments, we choose the Relu as the  activation function with $60$ hidden states and a linear output layer. It can be shown that the Lipschitz constant of Relu is $1$. Because the main objective is to classify the handwritten digits, we employ the cross-entropy loss function 
\[      \mathcal{L}(\,\by_t,\by_t^*\,) = -\by_t[{\bf i}^*]+\log \left(\sum_{j=1}^{m} \exp\left({\by_t\,[j]\,}\right) \right),\]
where ${\bf i}^* = \arg\max_{j} \SP \by_t^*\,[j]$ is a class label.  For stable training, we use the same approach suggested by \cite{miller2018stable}.  Hence, by performing a singular value decomposition for   $\A$ and clipping the singular values down to $1$, we can obtain a guaranteed stable model.

\begin{figure}[t]
    \centering
    \includegraphics[width=0.45\textwidth,trim= 55 30 70 60,clip]{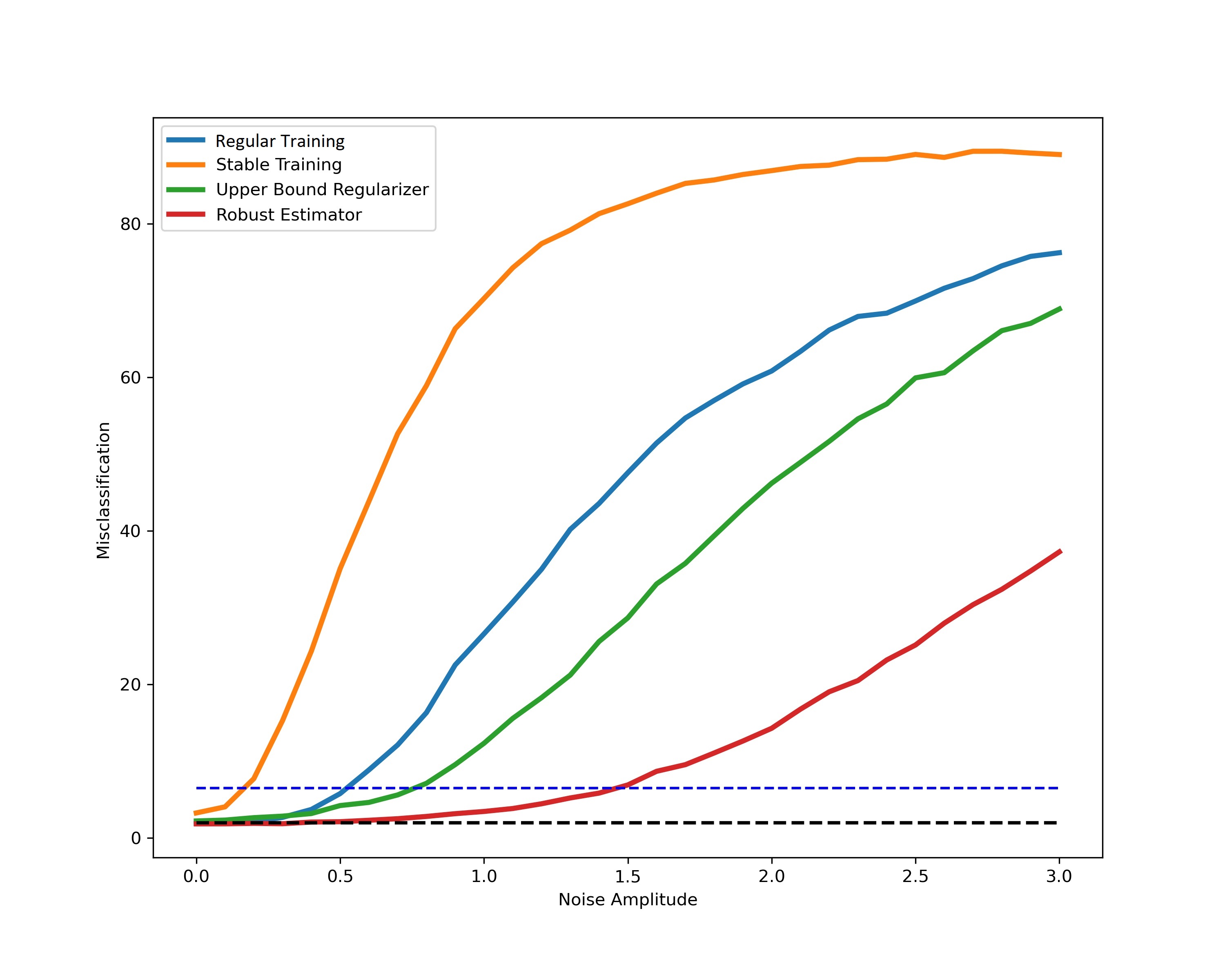}
    \caption{Misclassification percentage concerning the noise amplitude for different learning methods. }
    \label{fig:Rbst-plot}
\end{figure}


The learning accuracy curves for training and testing datasets for $200$ epochs are depicted in Figure \ref{fig:Learning curve}. Although the accuracy of the regular learning outperforms the other learning methods during the training phase,  the proposed robust learning methods achieve nearly the same level of classification precision, but they converge at a slower rate than the regular learning on the test dataset. Based on our simulations, the accuracy of classification by a basic RNN is at most $98.03 \%$ with regular learning for test dataset in the absence of noise.  As discussed earlier, imposing stability conditions during training will create gradient vanishing problems. Thus, as it is seen in Figure \ref{fig:Learning curve}, the final learning accuracy is comparably smaller that the other methods.

To analyze robustness properties of our proposed methods, each input sequence is distorted by a Gaussian noise with zero mean and constant covariance matrix $\omega I$. Figure \ref{fig:Noisy data} shows four different batches of $64$ samples for various noise intensities. One can observed that when $\omega > 2$, recognizing the true digits without help of a machine becomes challenging. 
Figure \ref{fig:Rbst-plot} illustrates the misclassification {percentage} with respect to the noise amplitude for the four different scenarios. The black dotted  line indicates the best achievable performance by a one-layer basic RNN  for the current choice of parameters. It is clear from Figure \ref{fig:Rbst-plot} that regularization via robustness estimator outperforms other methods significantly. One can also observe that the model learned by the robust architecture provides an almost constant precision for noise amplitude $\omega < 1$.

Table \ref{table:nonlin} shows noise amplitudes required to cause a certain level of misclassification for each method. Both proposed algorithms enhance robustness significantly. Regularization via robustness estimation reduces the missclassification in presence of exogenous noise least three times.  Learning using the upper-bound regularizer improves the  overall robustness by $50\%$ in comparison to the regular learning method.

Based on the discussion in Section \ref{Subsec-VI-2}  and the value of  $\norm{\A}$ displayed in the last row of Table \ref{table:nonlin}, regularization via upper-bound not only improves the overall robustness, but also steers the trainable parameters to be as close as possible to the stable region.  Our observations reveal that objective function \eqref{Loss-Omega} indirectly encourages stability without sacrificing learning accuracy. Finally, we highlight that none of the proposed robust learning methods guarantee stability, which suggest that  robust learning may not necessarily operate in stable regions.

\section{Conclusion}

We propose a formal analysis to study robustness properties of recurrent neural networks in presence of exogenous noise.  
Obtaining dynamic equations for noise propagation throughout time and quantifying the effect of noise on the output of recurrent networks have allowed us to introduce and utilize a proper measure of robustness to achieve robust learning. We have discussed and compared four different methods and shown that one can achieve significant robustness by regularizing the learning loss function with proper robustness measures. 
 



\begin{table}[t]
    \centering 
    \begin{tabular}{c |c c c c} 
        \toprule
        Percentage of & Stable & Regular & Upper Bound & Robustness \\ [0.5ex] 
        misclassification & Training & Training & Regularizer & Estimator \\ [0.5ex] 
        \midrule
        $3\,\%$ & 0.09 & 0.21 & 0.32 & 0.81 \\ 
        $5\, \%$ & 0.18 & 0.54 & 0.72 & 1.47 \\ 
        ${10}\, \%$ & 0.25 & 0.62 & 0.94 & 1.83 \\ 
        \midrule
        $\norm{\A}$ & 0.961 & 3.959 & 1.033 & 2.959 \\ 
        \bottomrule
    \end{tabular}
    \caption{The first three rows show that the required noise amplitude to incur a certain level of misclassification and the last row gives the norm of the resulting $\A$ for each method.} 
    \label{table:nonlin} 
\end{table}

\bibliography{Ref}

\end{document}